\documentclass[10pt,twoside]{IEEEtran}
\usepackage[nointegrals]{wasysym}
\usepackage{bbm,color,amstext,parskip}
\usepackage{booktabs,xcolor,siunitx}
\usepackage{url}
\usepackage[inline]{showlabels}
\usepackage{hyperref}
\usepackage[us,12hr]{datetime}

\usepackage{ulem}

\newcommand\redsout{\bgroup\markoverwith{\textcolor{red}{\rule[0.5ex]{2pt}{0.4pt}}}\ULon}

\usepackage{amsmath,amssymb,euscript,yfonts,latexsym}
\usepackage{amsthm}
\usepackage{algorithm}
\usepackage{algorithmic}
\usepackage{xcolor}

\newcommand{\bn}{\mathbf n}
\newcommand{\bo}{\mathbf o}

\newcommand{\by}{\mathbf y}

\newcommand{\cO}{\mathcal{O}}
\newcommand{\cX}{\mathcal{X}}

\newcommand{\mR}{\mathbb{R}}

\newtheorem{thm}{Theorem}

\newtheorem{prob}{Problem}
\newtheorem{rem}{Remark}
\usepackage{tikz}
\usepackage{graphicx,psfrag}
\usepackage{caption}
\usepackage{subcaption}
\usepackage{float}
\graphicspath{{./},{../figures/},{./figures/}}

\newcommand{\traj}{{\tau}}

\graphicspath{{./},{../figures/},{./figures/}}
\begin{document}
\title{Filtering for Aggregate Hidden Markov Models with Continuous Observations}
\author{Qinsheng Zhang, Rahul Singh, and Yongxin Chen
    \thanks{This work was supported by the NSF under grant 1901599, 1942523 and 2008513.
    }
    \thanks{Q. Zhang is with the Machine Learning Center, Georgia Institute of Technology, Atlanta, GA, USA. {\tt\small qzhang419@gatech.edu}. R. Singh, and Y. Chen are with the School of Aerospace Engineering, Georgia Institute of Technology, Atlanta, GA, USA. {\tt\small \{rasingh,yongchen\}@gatech.edu}}
}
\maketitle

\begin{abstract}
We consider a class of filtering problems for large populations where each individual is modeled by the same hidden Markov model (HMM). In this paper, we focus on aggregate inference problems in HMMs with discrete state space and continuous observation space. The continuous observations are aggregated in a way such that the individuals are indistinguishable from measurements. We propose an aggregate inference algorithm called continuous observation collective forward-backward algorithm. It extends the recently proposed collective forward-backward algorithm for aggregate inference in HMMs with discrete observations to the case of continuous observations. The efficacy of this algorithm is illustrated through several numerical experiments.
\end{abstract}

\section{Introduction}\label{sec:intro}

Hidden Markov Models~(HMMs) \cite{Mur12} are widely used in the systems and control community to model dynamical systems in areas such as robotics, navigation, and autonomy. An HMM has two major components, a Markov process that describes the evolution of the true state of the system and a measurement process corrupted by noise. One critical task in HMMs is to reliably estimate the state using the available noisy measurements, known as inference or filtering. Over the years, many filtering algorithms have been proposed. The most well-known one could be the Kalman Filter~\cite{WelBis95}, which is optimal for HMMs with Gaussian transition probability and Gaussian measurement noise. 
Several other well-known algorithms include the forward-backward algorithm \cite{Mur12} for discrete time discrete state HMMs and the Wonham algorithm \cite{Won64} for continuous time discrete state HMMs with Gaussian measurement noise.

Recently a new class of filtering problems for large populations with aggregate measurements has attracted much attention~\cite{SheDie11,HasRinChe19,SinHaaZha20}. 
In these filtering problems with aggregate observations, instead of individual measurements, only aggregate population-level data in the form of counts or contingency table is available~\cite{SheDie11}. 
Such scenarios may occur due to privacy or economic reasons.
For instance, in the study of animal flocking, it is too expensive, if not impossible, to track each individual. Aggregate measurements from surveillance cameras or other sensors are instead used. Another timely example is the modeling of pandemic. In such studies, the goal is to use available testing data to predict the evolution of the pandemic. Since the test results are often anonymized, they form aggregate measurements.

Filtering for large population is a difficult task. The lack of individual measurement makes it even more challenging, and most of the standard algorithms such as the forward-backward algorithm for HMMs are no longer applicable. 
A recent framework developed to address aggregate inference problems is the collective graphical models~(CGMs). Briefly, a CGM is a graphical model for the collective dynamics of a population.
Within the CGM framework, several algorithms have been proposed for aggregate inference including approximate MAP inference~\cite{SheSunKumDie13}, non-linear belief propagation~\cite{SunSheKum15} and Bethe-regularized dual averaging~\cite{VilBelShe15}. Several other recent works on filtering with aggregate observations include \cite{CheKar18,Zen19,SinHaaZha20,KimMeh20}.

In~\cite{HasRinChe19,SinHaaZha20} we proposed a novel algorithm, Sinkhorn belief propagation (SBP), for addressing aggregate filtering problems. It is built upon the connection \cite{HaaRinCheKar20,HasSinZhaChe20} between multi-marginal optimal transport (MOT)~\cite{Nen16,Pas12,BenCarCut15} and inference in probabilistic graphical models \cite{KolFri09} with fixed marginal constraints. When used in aggregate inference problems for HMMs, it is under the name collective forward-backward (CFB) algorithm to reflect its similarity and connections to the standard forward-backward algorithm. However, all the previous mentioned methods for aggregate inference only account for discrete observations. 
The goal of this work is to establish a counterpart of collective forward-backward (CFB) algorithm that is applicable to aggregate HMMs with continuous observations. 

In this work, we propose continuous observation collective forward-backward algorithm~(CO-CFB) to achieve this goal. 
The CO-CFB is closely related to the CFB algorithm for HMMs with discrete observations. The latter can be readily extended to HMMs with continuous observation but it leads to an algorithm which is not directly implementable. Rewriting a key step in it in terms of expectation circumvents this difficulty, and
we arrive at the CO-CFB that only uses samples from the aggregate observations.
Just like CFB, CO-CFB also exhibits convergence guarantees.
We remark that although CO-CFB is developed for HMMs, it can be extended to more general graphs.

The rest of the paper is organized as follows. In Section \ref{sec:background}, we review basis concepts in HMMs, aggregate inference and a previous algorithm for aggregate inference. In Section \ref{sec:main}, we present the CO-CFB algorithm, where several properties of it and its connections to other algorithms are also discussed. 
We demonstrate the efficacy of our algorithm using several numerical examples in Section \ref{sec:exp}. This is followed by a concise conclusion in Section \ref{sec:conclusion}.

\section{Background}
\label{sec:background}
In this section, we provide a brief introduction to HMMs and present existing results for aggregate inference in HMMs including the CFB algorithm.

\subsection{Hidden Markov Models} \label{subsec:hmm}
An HMM is a Markov process accompanied by observation noise. It consists of a sequence of unobservable states $X_t$ (also known as hidden states) and another sequence $O_t$ whose behavior depends on $X_t$. Here $X_t$ and $O_t$ are random variables taking values from $\cX$ and $\cO$ respectively.  In general, both $\cX$ and $\cO$ can be either finite sets or infinite sets. In this work, we assume $\cX$ to be a set with $d$ elements throughout. 

An HMM is characterized\footnote{Here we assume that the HMM is time-homogeneous, that is, transition and observation probabilities are time-invariant. This is just for notational convenience. All the results discussed in this paper apply to time-varying HMMs.} by the distribution $\pi(x_1)$ of the starting state $X_1$, the transition probability $p(x_{t+1} | x_t)$, and the observation model $p(o_t | x_t)$. The joint probability distribution over trajectories of $T$ steps reads
\begin{equation}
    \label{eq:hmm_p}
    p(\traj) = \pi(x_1) \prod_{t=1}^{T-1}p(x_{t+1}|x_t)\prod_{t=1}^Tp(o_t | x_t),
\end{equation}
where $x_t\in \cX$ denotes a realization of $X_t$, $o_t\in \cO$ denotes a realization of $O_t$, and $\traj=\{x_1,o_1,\cdots,x_t,o_t,\cdots,x_T,o_T\}$ denotes a sample trajectory.
A HMM can also be understood as a special case of probabilisitic graphical models \cite{KolFri09}; its graphical structure is depicted in Figure \ref{fig:single-hmm}. 
\begin{figure}[h]
        \centering
        \begin{tikzpicture}[scale=0.6,, transform shape,darkstyle/.style={circle,draw,fill=gray!30,minimum size=20}]

            \foreach \x in {1,...,1}
                {
                    \foreach \z in {1,...,4}
                        {
                            \node [darkstyle] (o\x\z) at (2*\z,0,0) {$o_{\z}$};
                            \node [circle,draw=black] (x\x\z) at (2*\z,2,0) {$x_{\z}$};
                            \path [draw,->] (x\x\z) edge (o\x\z);
                        }
                }
            \foreach \x in {1,...,1}
                {
                    \foreach \z in {1,...,3}
                        {\pgfmathtruncatemacro{\label}{\z+1}
                            \draw[->] (x\x\z)--(x\x\label);
                        }
                }
\end{tikzpicture}
\caption{Graphical representation of HMMs with $T=4$.}
\label{fig:single-hmm}
\end{figure}
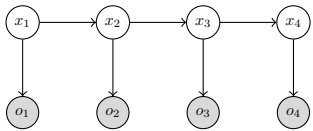

One of the most important problems in HMMs is Bayesian inference where the goal is to calculate the posterior distributions of the hidden states $X_t$ given a sequence of observations. This is also known as filtering~\cite{Mur12} in systems and control community.
A well-known and efficient algorithm for this task is the standard forward-backward algorithm (summarized in Algorithm \ref{alg:standard_forward_backward}). It relies on the following factorization 
\begin{align}
    \label{eq:fb-update}
    p(X_t = x_t|o_{1:T}) & \propto p({o}_{1:t-1} | x_t) p({o}_{t} | x_t) p({o}_{t+1:T} | x_t) p(x_t) \nonumber \\
                         & = p({o}_{t} | x_t) p({o}_{1:t-1}, x_t)  p({o}_{t+1:T} | x_t)
\end{align}
of the posterior distribution.
For fixed observations $o_{1:T}$, the quantities $\alpha_t(x_t)=p(x_t,{o}_{1:t-1})$ and $\beta_t(x_t) = p({o}_{t+1:T} | x_t)$ satisfy the recursive formulas 
\begin{subequations}\label{eq:forward_backwardst}
    \begin{eqnarray}
        \alpha_t(x_t) &=& \sum_{x_{t-1}} p(x_t|x_{t-1}) \alpha_{t-1}(x_{t-1}) p({o}_{t-1}|x_{t-1})
        \label{eq:forward_standard}\\
        \beta_t(x_t) &=& \sum_{x_{t+1}} p(x_{t+1}|x_t)  \beta_{t+1}(x_{t+1}) p({o}_{t+1}|x_{t+1})
        \label{eq:backward_standard}
    \end{eqnarray}
\end{subequations}
with boundary conditions $\alpha_1(x_1) = \pi(x_1)$ and $\beta_T(x_T) = 1$. The posterior distribution of $X_t$ is then calculated through
	\begin{equation}\label{eq:posterior}
		p(X_t = x_t|o_{1:T}) \propto \alpha_t(x_t)\beta_t(x_t)p(o_t | x_t).
	\end{equation}

\begin{algorithm}[t]
    \caption{Forward-Backward Algorithm}
    \label{alg:standard_forward_backward}
    \begin{algorithmic}
        \STATE Initialize the messages $\alpha_t(x_t), \beta_t(x_t)$ with $\alpha_1(x_1) = \pi(x_1)$ and $\beta_T(x_T) = 1$
        \STATE \textbf{Forward pass:}
        \FOR{$t = 2,3,\ldots,T$}
        \STATE Update $\alpha_t(x_t)$ using \eqref{eq:forward_standard}
        \ENDFOR
        \STATE \textbf{Backward pass:}
        \FOR{$t = T-1,\ldots,1$}
        \STATE Update $\beta_t(x_t)$ using \eqref{eq:backward_standard}
        \ENDFOR
    \end{algorithmic}
\end{algorithm}

\subsection{Aggregate Hidden Markov Models} \label{subsec:a-hmm}
Aggregate HMMs \cite{SunSheKum15,SinHaaZha20} deal with the problems of estimating the behavior of a collection of indistinguishable HMMs based on aggregate observations. Assume the observation space $\cO$ is a finite set. Given $M$ trajectories $\{\traj^{(1)},\cdots, \traj^{(M)}\}$ with $\traj^{(m)}=\{x_1^{(m)},o_1^{(m)},\cdots,x_T^{(m)},o_T^{(m)}\}$, sampled from the same HMM, the associated aggregate HMM is concerned with
the aggregate quantities $\bn=\{\bn_{t},\tilde{\bn}_{t},\bn_{tt},\tilde{\bn}_{tt}\}$ over the entire population defined by 
\begin{subequations}\label{eq:n_dist}
    \begin{align}
\tilde{n}_{tt}(x,o) = & \sum_{m=1}^M \mathbb{I}[x_t^{(m)}= x, o_t^{(m)}= o], ~ t \in \{1,\cdots, T\}\\
n_{tt}(x,x') = & \sum_{m=1}^M \mathbb{I}[x_t^{(m)}= x, x_{t+1}^{(m)}= x'], ~ t \in \{1,\cdots, T-1\}\\
n_t(x) = & \sum_{m=1}^{M} \mathbb{I}[x_t^{(m)}=x], ~ t \in \{1,\cdots,T\}\\
\tilde{n}_{t}(o) =& \sum_{m=1}^{M} \mathbb{I}[o_t^{(m)}=o],~ t \in \{1,\cdots,T\},
    \end{align}
\end{subequations}
where $\mathbb{I}$ denotes the indicator function. These quantities represent the counts of $M$ realizations of HMM taking specific values. Clearly, $\bn\in\mathbb{L}_M^{\mathbb{Z}}$, the \textit{integer-valued scaled local polytope} \cite{SunSheKum15}, that is, the entries of $\bn$ are all integers and they satisfy the following constraints
\begin{subequations}\label{eq:nconstraints}
    \begin{align}
& \sum_{o \in \cO} \tilde{n}_{t}(o) = M,\quad \sum_{x \in \cX} n_{t}(x)=M \quad\forall t \in \{1,\cdots, T\} \label{eq:M_cons_simplex} \\
& \sum_{x \in \cX} n_{tt}(x,x_{t+1})=n_{t+1}(x_{t+1}), \nonumber \\ 
& \sum_{x \in \cX} n_{tt}(x_t,x)=n_t(x_t), \quad\forall t \in \{1,\cdots, T-1\}  \label{eq:cons_tt_t} \\
& \sum_{o \in \cO} \tilde{n}_{tt}(x,o)=n_t(x), \quad\forall t \in \{1,\cdots, T\}                \label{eq:cons_tilde_tt_t}                   \\
         & \sum_{x \in \cX} \tilde{n}_{tt}(x,o)=\tilde{n}_{t}(o), \quad\forall t \in \{1,\cdots,T\}\label{eq:cons_tilde_tt_o}.
    \end{align}
\end{subequations}
The HMMs are aggregate in the sense that they are indistinguishable to each other. In this setting, the observations\footnote{A slightly different observation has been considered in \cite{SunSheKum15}.} are $\by_t$ with $y_t(o)$ denoting the number of trajectories such that $o_t^{(m)} = o$, imposing the constraints $\tilde \bn_t = \by_t$ for $t=1,\ldots,T$. 
An example of aggregate HMM is illustrated in Figure \ref{fig:aggregate-hmm}.
\begin{figure}
        \centering
        \begin{tikzpicture}[scale=0.55,, transform shape,darkstyle/.style={circle,draw,fill=gray!30,minimum size=20}]
            \foreach \z in {1,...,4}
                {
                    \node [circle,fill=blue!30] (y\z) at (8,-3,-2*\z) {$\by_{\z}$};
                    \node [circle,fill=red!50] (n\z) at (0,4,-2*\z) {$\bn_{\z}$};
                }

            \foreach \x in {1,...,4}
                {
                    \draw[fill=gray!05] (2*\x,0-1,1) -- (2*\x, 0-1,-9) -- (2*\x, 2+1,-9) -- (2*\x, 2+1,1) -- cycle;
                    \foreach \z in {1,...,4}
                        {
                            \node [darkstyle] (o\x\z) at (2*\x,0,-2*\z) {$o_{\z}^{(\x)}$};
                            \node [circle,draw=black] (x\x\z) at (2*\x,2,-2*\z) {$x_{\z}^{(\x)}$};
                            \path [draw,->] (x\x\z) edge (o\x\z);
                        }
                    \draw[->,bend right] (o\x1) to (y1);
                    \draw[->,bend right] (x\x1) to (n1);
                }

            \foreach \x in {1,...,4}
                {
                    \foreach \z in {1,...,2}
                        {
                            \draw[->,bend right] (o\x\z) to (y\z);
                            \draw[->,bend right] (x\x\z) to (n\z);
                        }
                }

            \foreach \x in {1,...,4}
                {
                    \foreach \z in {1,...,3}
                        {\pgfmathtruncatemacro{\label}{\z+1}
                            \draw[->] (x\x\z)--(x\x\label);
                        }
                }

            \foreach \x in {3,...,8}
                {
                    \node[mark size=1pt,color=black] at (8,-2,-\x) {\pgfuseplotmark{*}};
                    \node[mark size=1pt,color=black] at (0,3.3,-\x -2) {\pgfuseplotmark{*}};
                }
\end{tikzpicture}
\caption{An example of aggregate HMM with $M=4$. Each vertical plane stands for one sample trajectory from the underlying HMM. 
The blue nodes represent aggregate observations, and red nodes represent aggregate hidden states.}
 \label{fig:aggregate-hmm}
\end{figure}
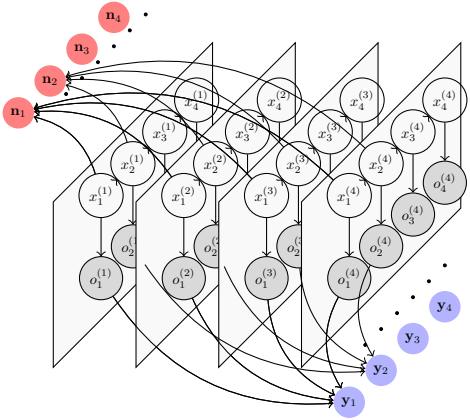

The goal of inference in aggregate HMMs is to estimate the most likely $\bn$ given the aggregate measurements \{$\by_t$\}. The exact inference is proved to be computationally infeasible~\cite{SheDie11} for problems with large $T$ and $M$. 
It is proposed in \cite{SinHaaZha20} that this aggregate inference can be approximately achieved by solving a free energy minimization problem. Moreover, the approximation error vanishes as the size $M$ of the population goes to infinity. 
For the sake of simplicity and without loss of generality, we normalize the observation $\by_t$ as well as the statistics $\bn$ by population size $M$, yielding a modification on Equation~\eqref{eq:M_cons_simplex}
\begin{equation}
    \label{eq:cons_simplex}
    \sum_{o \in \cO} \tilde{n}_{t}(o) = 1,\quad \sum_{x \in \cX} n_{t}(x)=1 \quad\forall t \in \{1,\cdots, T\}.
\end{equation}
Denote the local polytope (without integer constraints) described by Equation~\eqref{eq:cons_simplex}-\eqref{eq:cons_tt_t}-\eqref{eq:cons_tilde_tt_t}-\eqref{eq:cons_tilde_tt_o} by $\mathbb{M}$~\cite{WaiJor08}, and define the free energy term
\begin{align} \label{eq:bethe}
    \mathcal{F}(\bn) = 
    &- \sum_{t=1}^T \sum_{x_t,o_t} \tilde{n}_{tt}(x_t,o_t) \log p(o_t | x_t) \\ \nonumber
    &- \sum_{t=1}^{T-1} \sum_{x_t,x_{t+1}} {n}_{tt}(x_t,x_{t+1}) \log p(x_{t+1} | x_t)  \\ \nonumber
    &- \sum_{x_1} n_{1}(x_1) \log \pi(x_1) -\sum_{x_1} n_1(x_1)\log n_1(x_1) \\ \nonumber
    & + \sum_{t=1}^T \sum_{x_t,o_t} \tilde{n}_{tt}(x_t,o_t) \log \tilde{n}_{tt}(x_t,o_t) \\ \nonumber
    &+ \sum_{t=1}^{T-1} \sum_{x_t,x_{t+1}} {n}_{tt}(x_t,x_{t+1}) \log {n}_{tt}(x_t,x_{t+1})\\ \nonumber
    & - 2\sum_{t=2}^{T-1}\sum_{x_t} n_t(x_t) \log n_t(x_t) -\sum_{x_T}n_T(x_T)\log n_T(x_T),
\end{align}
then the aggregate inference problem is equivalent \cite{SinHaaZha20} to the following convex optimization problem
\begin{prob}\label{prob:bethe}
    \begin{subequations} \label{eq:argmin-bethe}
        \begin{eqnarray}
            \min_{\bn\in \mathbb{M}} && \mathcal{F}(\bn)\\
            \text{ subject to} && \tilde{\bn}_t = \by_t, \quad \forall t\in\{1,\cdots,T\}.
        \end{eqnarray}
    \end{subequations}
\end{prob}

\subsection{Collective forward-backward algorithm} \label{subsec:afb}

An efficient algorithm known as collective forward-backward~(CFB) algorithm (Algorithm \ref{alg:forward_backward}) was proposed in \cite{SinHaaZha20} to address Problem \ref{prob:bethe}. It leverages an elegant connection between multi-marginal optimal transport~\cite{Nen16,Pas12,BenCarCut15} and this inference problem with fixed marginal constraints. We refer the reader to~\cite{SinHaaZha20} for more details on Collective Forward-backward algorithm and its connections to multi-marginal optimal transport. The solution to Problem \ref{prob:bethe} is characterized by the following theorem.
\begin{thm}[{\cite[Corollary 1]{SinHaaZha20}}] \label{thm:sbp_hmm}
    The solution to the filtering problem (Problem \ref{prob:bethe}) for aggregate HMM is
    \begin{equation}\label{eq:statemarginals}
        n_t(x_t) \propto \alpha_t(x_t) \beta_t(x_t) \gamma_t(x_t),~~ \forall t =1,2\ldots,T
    \end{equation}
    where $\alpha_t(x_t), \beta_t(x_t)$, $\gamma_t(x_t)$, and $\xi_t(o_t)$ are the fixed points of the following updates
    \begin{subequations}\label{eq:d_forward_backward}
        \begin{eqnarray}
            \alpha_t(x_t) = \sum_{x_{t-1}} p(x_t|x_{t-1}) \alpha_{t-1} (x_{t-1}) \gamma_{t-1}(x_{t-1}), \label{eq:d_forward_backward1}  \\
            \beta_t(x_t) = \sum_{x_{t+1}} p(x_{t+1}|x_{t}) \beta_{t+1} (x_{t+1}) \gamma_{t+1}(x_{t+1}), \label{eq:d_forward_backward2} \\
            \gamma_t(x_t) = \sum_{o_{t}} p(o_{t}|x_{t}) \frac{y_t(o_t)}{\xi_t(o_t)}, \label{eq:d_forward_backward3} \\
            \xi_t(o_t) = \sum_{x_{t}} p(o_{t}|x_{t}) \alpha_{t} (x_{t}) \beta_{t}(x_{t}), \label{eq:d_forward_backward4}
        \end{eqnarray}
    \end{subequations}
    with boundary conditions
    \begin{equation}
        \alpha_1(x_1) = \pi(x_1) \quad \text{and} \quad \beta_T(x_T) = 1.
    \end{equation}
\end{thm}
\begin{algorithm}[h]
   \caption{Collective Forward-Backward (CFB) Algorithm}
   \label{alg:forward_backward}
\begin{algorithmic}
   \STATE Initialize all the messages $\alpha_t(x_t), \beta_t(x_t), \gamma_t(x_t), \xi_t(o_t)$ 
   \WHILE{not converged}
   \STATE \textbf{Forward pass:}
   \FOR{$t = 2,3,\ldots,T$}
        \STATE i) Update  $\gamma_{t-1}(x_{t-1})$ 
        \STATE ii) Update $\alpha_t(x_t), \xi_t(o_t)$
    \ENDFOR
    \STATE \textbf{Backward pass:}
    \FOR{$t = T-1,\ldots,1$}
        \STATE i) Update  $\gamma_{t+1}(x_{t+1})$
        \STATE ii) Update $\beta_t(x_t), \xi_t(o_t)$
    \ENDFOR
    \ENDWHILE
\end{algorithmic}
\end{algorithm}
Figure \ref{fig:alg-afb} illustrates $\alpha,\beta,\gamma, \xi$ from Theorem \ref{thm:sbp_hmm} in an HMM. By comparing \eqref{eq:forward_backwardst}-\eqref{eq:posterior} with \eqref{eq:statemarginals}-\eqref{eq:d_forward_backward}, it is easy to see that the Forward-backward algorithm (Algorithm \ref{alg:standard_forward_backward}) is a special case of the Collective Forward-backward algorithm for one trajectory, namely, $M=1$. 
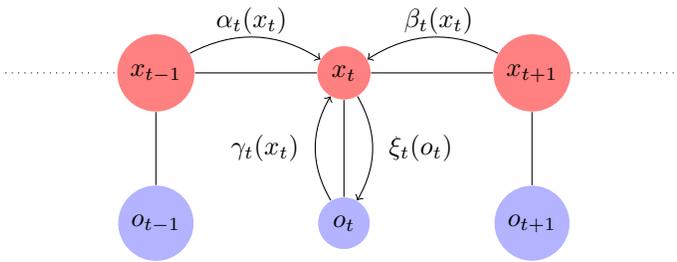
\begin{figure} \label{fig:hmm-message}
    \centering
    \begin{tikzpicture}[scale=1.0,, transform shape,darkstyle/.style={circle,draw,fill=gray!30,minimum size=30,maximum size=30}]
        \node [circle,fill=red!50] at (2,0) (b2) {$x_{t-1}$};
        \node [circle,fill=red!50] at (4.5,0) (b3) {$x_t$};
        \node [circle,fill=red!50] at (7,0) (b4) {$x_{t+1}$};

        \draw[dotted] (0,0) to (b2);
        \draw[solid] (b2) to (b3);
        \draw[solid] (b3) to (b4);
        \draw[dotted] (b4) to (9,0);

        \node [circle,fill=blue!30] at (2,-2) (o2) {$o_{t-1}$};
        \node [circle,fill=blue!30] at (4.5,-2) (o3) {$o_t$};
        \node [circle,fill=blue!30] at (7,-2) (o4) {$o_{t+1}$};

        \draw[solid] (b2) to (o2);
        \draw[solid] (b3) to (o3);
        \draw[solid] (b4) to (o4);

        \draw[->, bend left] (b3) to (o3);
        \draw[->, bend left] (o3) to (b3);

        \node[text width=1 cm,  black] at (5.6, -1.0) {$\xi_t(o_t)$};
        \node[text width=1 cm,  black] at (3.5, -1.0) {$\gamma_t(x_t)$};

        \draw[->, bend left] (b2) to (b3);
        \draw[->, bend right] (b4) to (b3);

        \node[text width=1 cm,  black] at (3.3, 0.7) {$\alpha_t(x_t)$};
        \node[text width=1 cm,  black] at (5.8, 0.7) {$\beta_t(x_t)$};
    \end{tikzpicture}
    \caption{Illustration of the CFB algorithm.}
    \label{fig:alg-afb}
\end{figure}

\section{Main Results}\label{sec:main}

We consider filtering problems for aggregate HMMs with continuous observations. More specifically, we consider an inference problem similar to Problem \ref{prob:bethe} but assume the observation space of the underlying HMM to be a subset of the Euclidean space, i.e., $\cO\subset\mR^\ell$. An important instance of this observation model corresponds to the Gaussian measurement noise, that is, $p(o_t | x_t) = \mathcal{N}(\mu(x_t),\Sigma(x_t))$ with $\mathcal{N}(\mu,\Sigma)$ denoting the probability density with mean $\mu$ and covariance $\Sigma$.

\subsection{Aggregate inference with Continuous observations}\label{sec:continuous}
We start with an ideal setting when the full distributions of the observations are given. More precisely, we assume the observations are the full distributions $y_t(o_t),\,o_t\in \cO$. In this scenario, the aggregate inference can again be formulated by Problem \ref{prob:bethe} with two modifications due to the continuous observations. First, the constraints \eqref{eq:cons_simplex}-\eqref{eq:cons_tilde_tt_t} become
\begin{subequations}
    \begin{align}
        \int_{o \in \cO} \tilde{n}_{t}(o)do = 1,\quad \sum_{x \in \cX} n_{t}(x)=1 \quad\forall t \in \{1,\cdots, T\} \label{eq:update_cons_simplex} \\
        \int_{o \in \cO} \tilde{n}_{tt}(x,o) do=n_t(x) \quad\forall t \in \{1,\cdots,T\}\label{eq:update_cons_tilde_tt_t}.
    \end{align}
\end{subequations}
Thus, the constraint set $\mathbb{M}$ is described by \eqref{eq:update_cons_simplex}-\eqref{eq:cons_tt_t}-\eqref{eq:update_cons_tilde_tt_t}-\eqref{eq:cons_tilde_tt_o}. Second, the free energy objective becomes
\begin{align}\label{eq:new_bethe}
    \mathcal{F}(\bn) = 
    &- \sum_{t=1}^T \sum_{x_t} \int_{\cO} \tilde{n}_{tt}(x_t,o_t) \log p(o_t | x_t) do_t \\ \nonumber
    &- \sum_{t=1}^{T-1} \sum_{x_t,x_{t+1}} {n}_{tt}(x_t,x_{t+1}) \log p(x_{t+1} | x_t) \\ \nonumber
    & + \sum_{t=1}^T \sum_{x_t} \int_{\cO} \tilde{n}_{tt}(x_t,o_t) \log \tilde{n}_{tt}(x_t,o_t) do_t \\ \nonumber
    &+ \sum_{t=1}^{T-1} \sum_{x_t,x_{t+1}} {n}_{tt}(x_t,x_{t+1}) \log {n}_{tt}(x_t,x_{t+1})\\ \nonumber
    &-\!\! 2\!\!\sum_{t=2}^{T-1}\sum_{x_t} n_t(x_t) \log n_t(x_t) \!\! - \!\!\! \sum_{x_1} n_1(x_1)\log n_1(x_1) \\ \nonumber
    &-\sum_{x_T}n_T(x_T)\log n_T(x_T) - \sum_{x_1} n_{1}(x_1) \log \pi(x_1) .
\end{align}

It turns out that the solution to this aggregate inference problem with continuous observation has a similar characterization as in the discrete setting as in the following Theorem. 
\begin{thm}\label{thm:co-forward-backward}
    The solution to aggregate filtering problem (Problem \ref{prob:bethe}) in an aggregate HMM with continuous observation is
        \begin{equation}\label{eq:marginals_hmm}
            n_t(x_t) \propto \alpha_t(x_t) \beta_t(x_t) \gamma_t(x_t),~~ \forall t =1,\ldots,T 
        \end{equation}
    where $\alpha_t(x_t), \beta_t(x_t),$ and $\gamma_t(x_t)$ are the fixed points of the following updates
    \begin{subequations}\label{eq:forward_backward}
        \begin{eqnarray}
            \alpha_t(x_t) &=& \sum_{x_{t-1}} p(x_t|x_{t-1}) \alpha_{t-1} (x_{t-1}) \gamma_{t-1}(x_{t-1}), \label{eq:forward_backward1} \\
            \beta_t(x_t) &=& \sum_{x_{t+1}} p(x_{t+1}|x_{t}) \beta_{t+1} (x_{t+1}) \gamma_{t+1}(x_{t+1}), \label{eq:forward_backward2} \\
            \gamma_t(x_t) &=& \int_{\cO} p(o_t|x_t)\frac{y_t(o_t)}{\xi_t(o_t)}d{o_t}, \label{eq:forward_backward3} \\
            \xi_t(o_t) &=& \sum_{x_{t}} p(o_{t}|x_{t}) \alpha_{t} (x_{t}) \beta_{t}(x_{t}), \label{eq:forward_backward4}
        \end{eqnarray}
    \end{subequations}
    with boundary conditions
    \begin{equation*}
        \alpha_1(x_1) = \pi(x_1) \quad \text{and} \quad \beta_T(x_T) = 1.
    \end{equation*}
\end{thm}
\begin{proof}
See Appendix~\ref{app:theorem2}
\end{proof}

\subsection{Continuous Observation Collective Forward-Backward Algorithm}

In Section \ref{sec:continuous}, we assumed that the full distributions of the observations are given, which is hardly the case in real applications. Very often only samples from these marginal distributions are accessible. Of course, one can always estimate a probability density based on these samples and then run the updates \eqref{eq:forward_backward}. The performance of this approach highly depends on that of the density estimators. Another reason that makes this approach undesirable is that the step \eqref{eq:forward_backward3} requires numerical integration which is expensive when the dimension of observation space $\cO$ is large.

A better approach is to rewrite the step \eqref{eq:forward_backward3} as
	\begin{equation}
		\gamma_t(x_t)=\mathbb{E}_{o_t}\left[\frac{p(o_t|x_t)}{\xi_t(o_t)}\right].
	\end{equation}
Let $\bo_t = \{o_t^{(1)},\cdots, o_t^{(M)}\}$ be the aggregate observation at time $t$, then this update formula for $\gamma_t(x_t)$ can be estimated by
\begin{equation}\label{eq:forward_backward3_updated}
    \gamma_t(x_t) = \sum_{o \in \bo_t} \frac{p(o|x_t)}{\xi_t(o)}.
\end{equation}

With this update in hand, we propose continuous observation collective forward-backward~(CO-CFB) algorithm (Algorithm \ref{alg:main}) for solving aggregate inference in aggregate HMMs with continuous observations. 
\begin{algorithm}[h]
    \caption{Continuous Observation Collective Forward-Backward~(CO-CFB) Algorithm}
    \label{alg:main}
    \begin{algorithmic}
        \STATE Initialize all the messages $\alpha_t(x_t), \beta_t(x_t), \gamma_t(x_t), \xi_t(o_t)$
        \WHILE{not converged}
        \STATE \textbf{Forward pass:}
        \FOR{$t = 2,3,\ldots,T$}
        \STATE i) Update  $\gamma_{t-1}(x_{t-1})$ using \eqref{eq:forward_backward3_updated}
        \STATE ii) Update $\alpha_t(x_t), \xi_t(o_t)$ according to \eqref{eq:forward_backward1}, \eqref{eq:forward_backward4}
        \ENDFOR
        \STATE \textbf{Backward pass:}
        \FOR{$t = T-1,\ldots,1$}
        \STATE i) Update  $\gamma_{t+1}(x_{t+1})$ using \eqref{eq:forward_backward3_updated}
        \STATE ii) Update $\beta_t(x_t), \xi_t(o_t)$ according to \eqref{eq:forward_backward2}, \eqref{eq:forward_backward4}
        \ENDFOR
        \ENDWHILE
    \end{algorithmic}
\end{algorithm}

Algorithm \ref{alg:main} can be viewed as a counterpart of the Collective forward-backward algorithm (Algorithm \ref{alg:forward_backward}) for aggregate HMMs with continuous observations. It is almost the same as the latter except for a small difference in the update step for $\gamma_t$, replacing \eqref{eq:d_forward_backward3} by \eqref{eq:forward_backward3_updated}.  For a given observation $\bo_t = \{o_t^{(1)},\cdots, o_t^{(M)}\}$, if we restrict the observation space to $\bo_t$, then the aggregate HMM with continuous observation reduces to an aggregate HMM with discrete observation over the set $\bo_t$. Since all the samples in $\bo_t$ are equally important, the corresponding probability vector measurement $y_t(o_t)$ should be uniform over $\bo_t$. That's why it doesn't appear in \eqref{eq:forward_backward3_updated}. Thus, to some extent, CO-CFB and CFB are equivalent.

\begin{rem}
An interesting special case of HMMs with continuous observations is the aggregate Gaussian mixture models (GMMs). To see this, simply take $T=1$ and let $p(o_1|x_1) = \mathcal{N}(\mu(x_1), \Sigma(x_1))$ be a Gaussian distribution for each $x_1\in \cX$. In this case, we assume the aggregate measurement is available and we are interested in the distribution of the hidden variable $\bn_1$. By setting $T=1$ in Algorithm \ref{alg:main}, we arrive at a closed form solution to aggregate inference problem for GMMs, which reads $$n_1(x_1) \propto \pi(x_1) \sum_{o \in \bo_1} \frac{p(o|x_1)}{\sum_{x_1} p(o|x_1)\pi(x_1)}.$$
Clearly, when $M=1$, this is nothing but the posterior distribution of the latent variable in a GMM.
\end{rem}

\subsection{Connections to Bayesian inference of HMMs}
As discussed in Section \ref{subsec:afb}, for HMMs with discrete observations, the aggregate inference reduces to standard Bayesian inference when the population size is $M=1$, and the CFB algorithm (Algorithm \ref{alg:forward_backward}) reduces to the standard forward-backward algorithm (Algorithm \ref{alg:standard_forward_backward}). It turns out that this equivalent relation remains for HMMs with continuous observations. We remark that the Forward-backward algorithm for HMMs with continuous observation is also known as Wonham filtering \cite{Won64}. 
\begin{thm}
When the population size is ($M=1$), the CO-CFB algorithm (Algorithm \ref{alg:main}) reduces to the standard forward-backward algorithm (Algorithm \ref{alg:standard_forward_backward}).
\end{thm}

\begin{proof}
When $M=1$, the observation at time $t$ is $\bo_t=\{o_t^{(1)}\}$. The update \eqref{eq:forward_backward3_updated} for $\gamma$ becomes
	\[
		\gamma_t(x_t) = \frac{p(o_t^{(1)}|x_t)}{\xi_t(o_t^{(1)}}\propto p(o_t^{(1)}|x_t).
	\]
With this $\gamma_t$, the update rules \eqref{eq:forward_backward1} and \eqref{eq:forward_backward2} take the form
    \begin{subequations}\label{eq:standard_forward_backward}
        \begin{eqnarray}
            \alpha_t(x_t) = \sum_{x_{t-1}} p(x_t|x_{t-1}) \alpha_{t-1} (x_{t-1}) p(o^{(1)}_{t-1} | x_{t-1}), \label{eq:standard_forward_backward1}  \\
            \beta_t(x_t) = \sum_{x_{t+1}} p(x_{t+1}|x_{t}) \beta_{t+1} (x_{t+1}) p(o^{(1)}_{t+1} | x_{t+1}), \label{eq:standard_forward_backwar2}
        \end{eqnarray}
    \end{subequations}
and the expression for the marginal distributions $\bn_t$ becomes
    \begin{equation} \label{eq:proof-marginal}
        n_t(x_t) \propto \gamma_t(x_t)\alpha_t(x_t) \beta_t(x_t) = p(o^{(1)}_t | x_{t}) \alpha_t(x_t) \beta_t(x_t).
    \end{equation}
These are exactly the same as the Forward-backward algorithm in \eqref{eq:forward_backwardst}-\eqref{eq:posterior}.
\end{proof}

\begin{figure}[h]
        \centering
    \begin{subfigure}[b]{0.38\textwidth}
        \centering
        \includegraphics[width=0.8\textwidth]{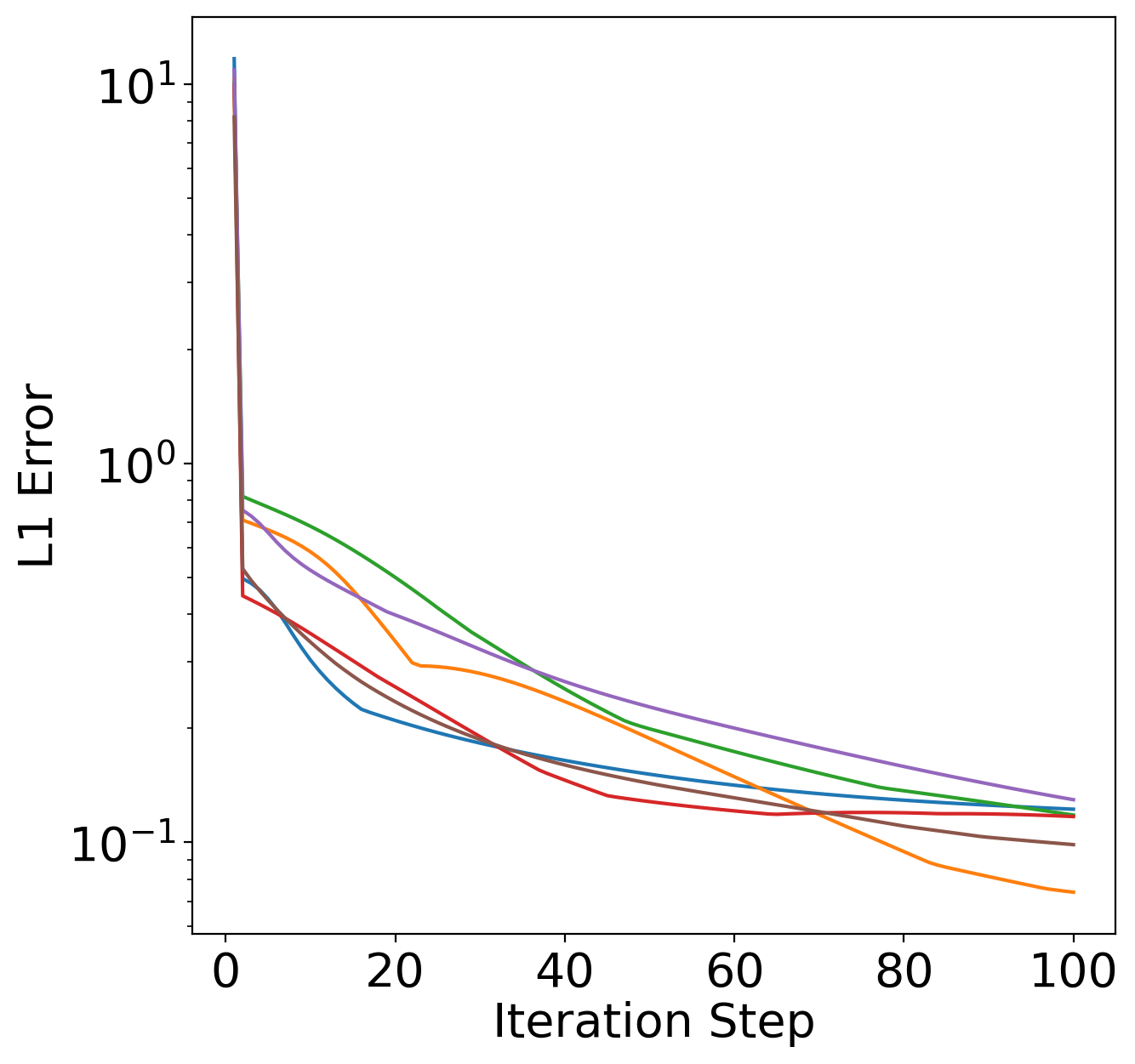}
        \caption{$M=200$}
        \label{fig:error_M200}
    \end{subfigure}
    \begin{subfigure}[b]{0.38\textwidth}
        \centering
        \includegraphics[width=0.8\textwidth]{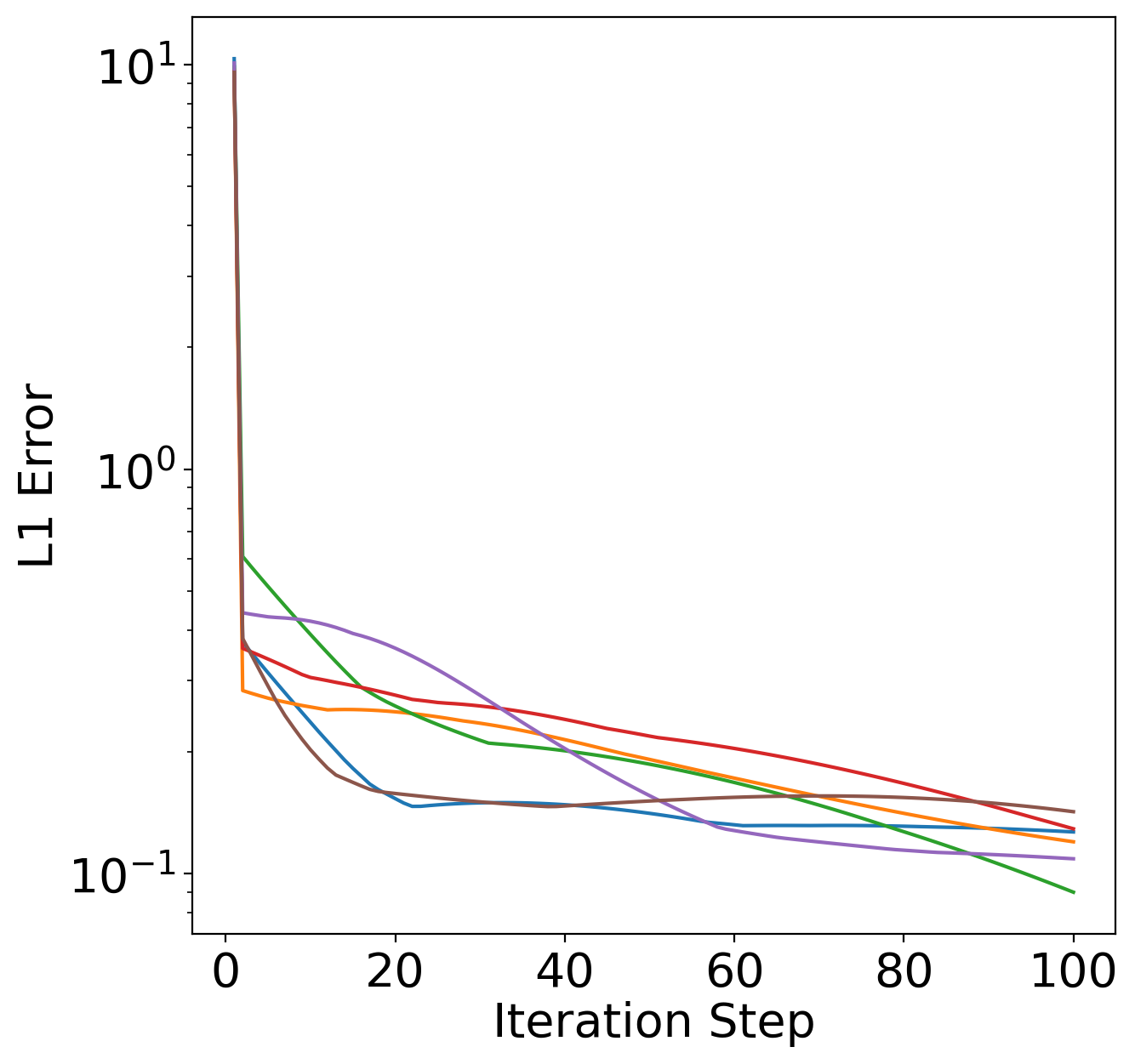}
        \caption{$M=500$}
        \label{fig:error_M500}
    \end{subfigure}  
        \caption{Convergence of estimation error for different population size $M$. The HMM length and number of state are fixed to be $T=20, d=20$. The curves in different color represent results with different randomly generated initial probability, transition probability and emission probability.}
        \label{fig:error}
 \end{figure}
\section{Numerical Examples} \label{sec:exp}
We conduct several synthetic experiments to evaluate the performance of our CO-CFB algorithm. 
We consider HMMs with Gaussian observations. The initial state probability $\pi$ is sampled uniformly over the probability simplex. The transition matrix is generated from a random permutation of a perturbed Identity matrix $\mathcal{I} + 0.05 \times \sqrt{d} \times \exp(Uniform[-1,1])$. A normalization is needed to ensure each row is a valid conditional distribution. The Gaussian emission probability for each hidden state is parameterized by a random mean and variance. The mean is sampled from $Uniform[-5d,5d]$ and the variance is sampled from $Uniform[1,5]$. 

We first demonstrate that our algorithm gives excellent estimation of the distributions of the hidden states of aggregate HMMs. We randomly generate $M$ trajectories from the underlying HMM and produce the observations $\bo_t,~t=1,\ldots, T$ based on these trajectories. We then use CO-CFB to estimate the hidden state distribution $\bn_t$. This is compared to the ground truth $\bn_t^*$ with respect to 1-norm
\[
    Error = \sum_{t=1}^T \| \bn_t - \bn^*_t\|_1.
\]
 Figure \ref{fig:error} depicts the estimation error for two different population size. Clearly, the estimation error goes to very small values in both cases.

To evaluate the efficiency of our algorithm, we test it over different state dimension $d$ and HMM length $T$. 
We terminate the algorithm when the relative change of error is less than $1\times10^{-5}$. 
The result is displayed in Figure \ref{fig:comparision}, from which we observe that the CO-CFB algorithm has great scalability. We also validate the linear dependency of computation complexity on population size $M$ in Figure \ref{fig:comparision_M}.

\begin{figure}[h]
    \centering
    \begin{subfigure}[b]{0.24\textwidth}
        \centering
        \includegraphics[width=1.0\textwidth]{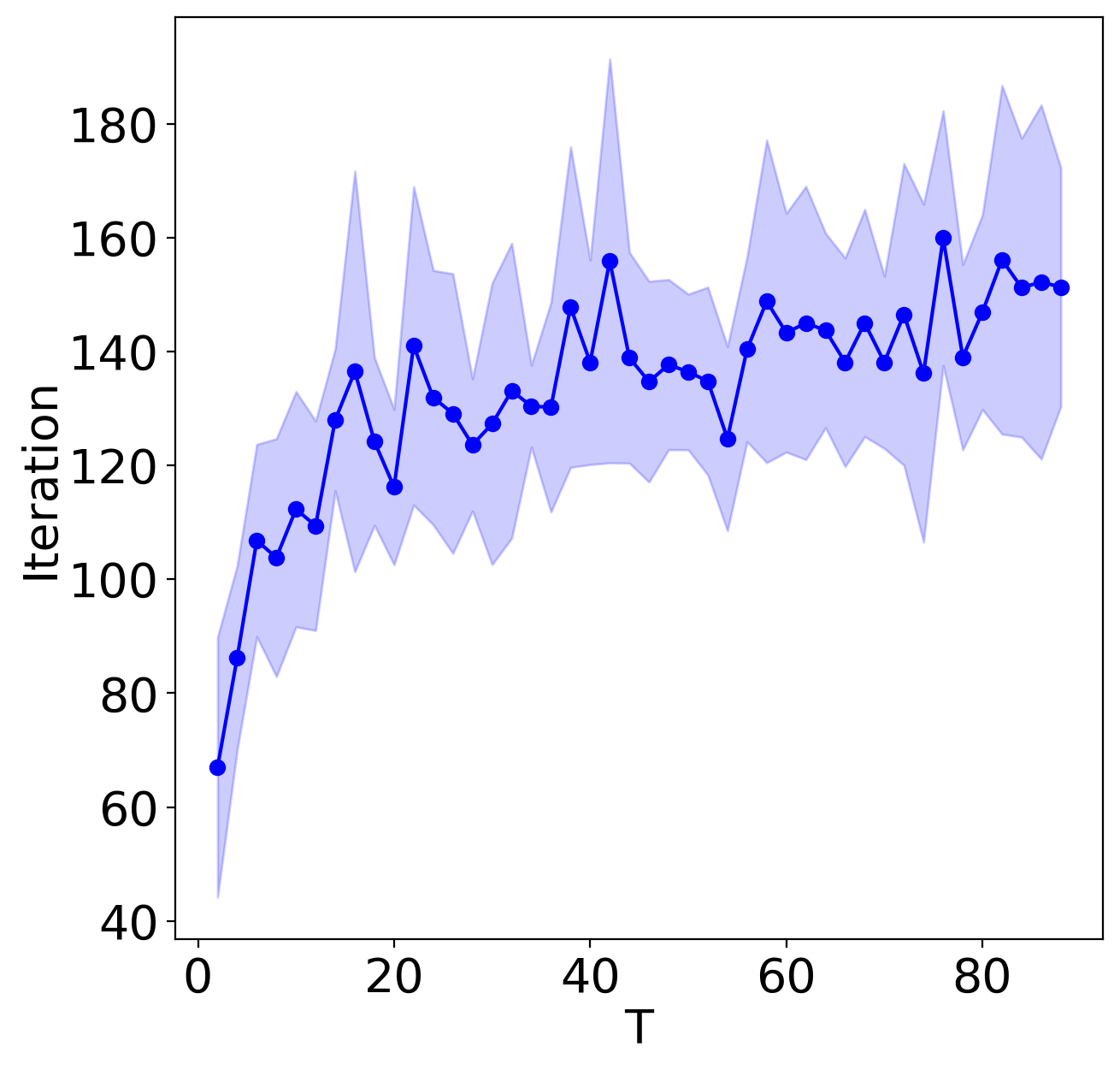}
        \caption{}
        \label{fig:length_step}
    \end{subfigure}
    \begin{subfigure}[b]{0.24\textwidth}
        \centering
        \includegraphics[width=1.0\textwidth]{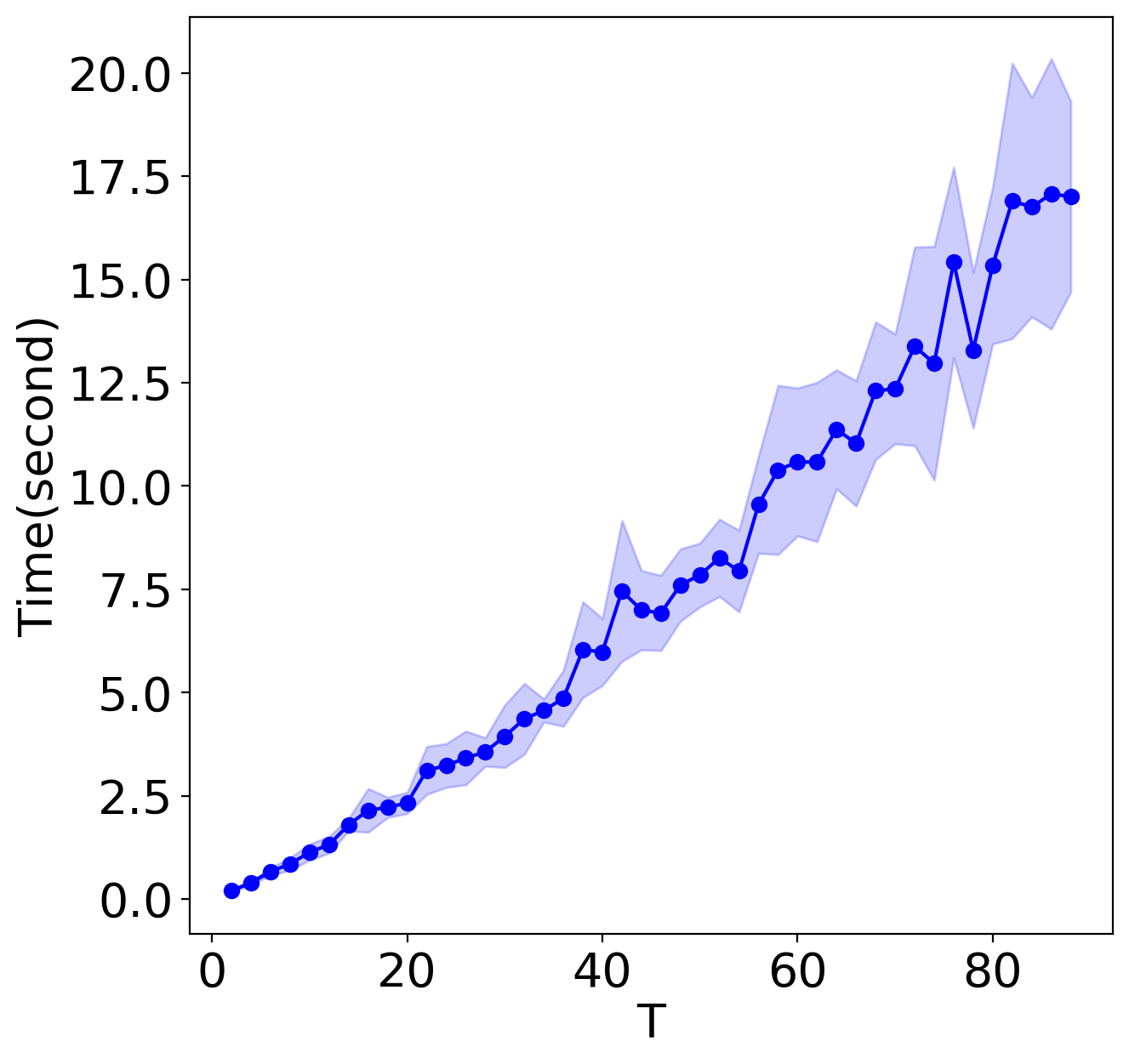}
        \caption{}
        \label{fig:length_time}
    \end{subfigure}

    \begin{subfigure}[b]{0.24\textwidth}
        \centering
        \includegraphics[width=1.0\textwidth]{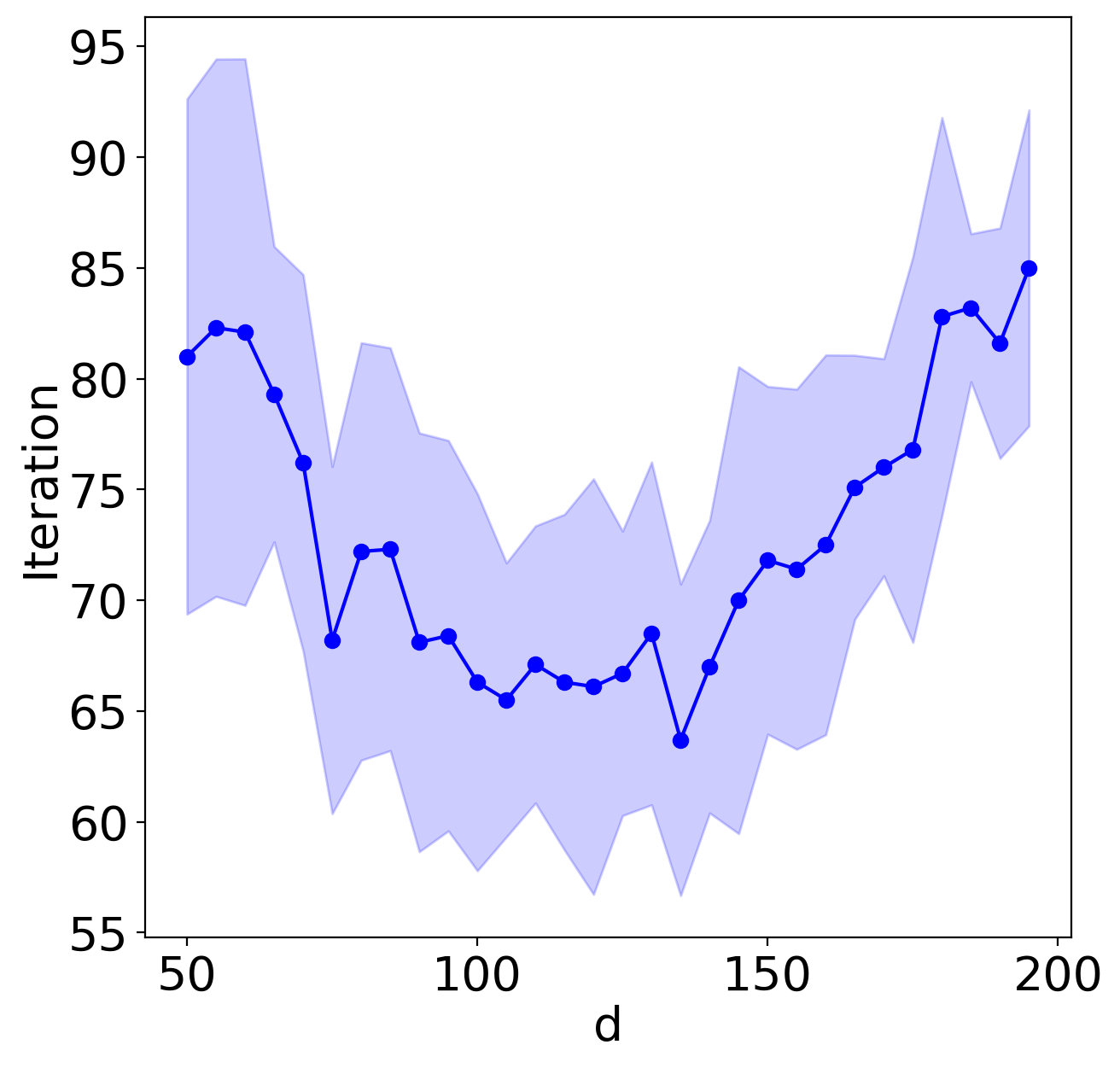}
        \caption{}
        \label{fig:state_step}
    \end{subfigure}
    \begin{subfigure}[b]{0.24\textwidth}
        \centering
        \includegraphics[width=1.0\textwidth]{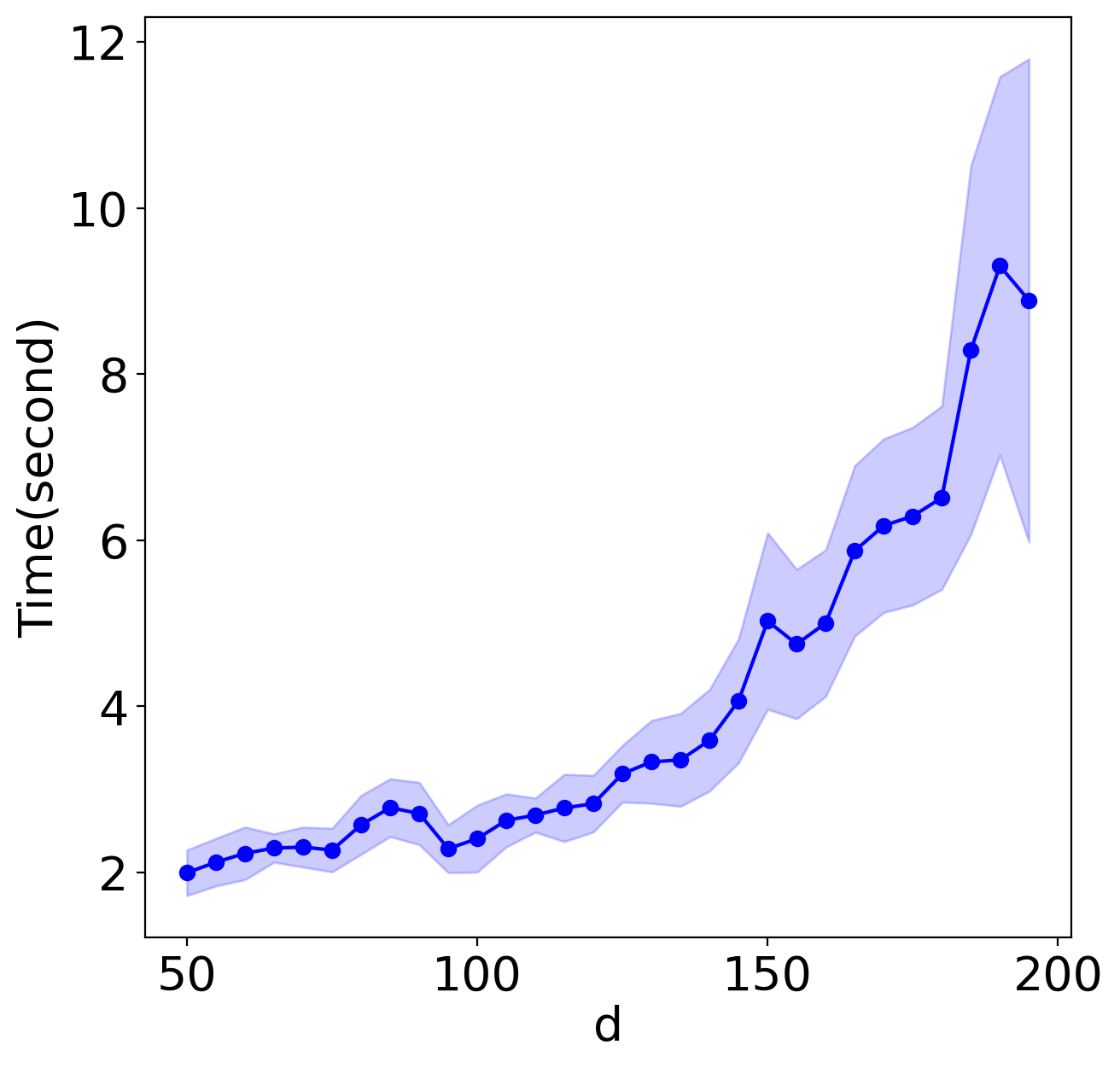}
        \caption{}
        \label{fig:state_time}
    \end{subfigure}
    \caption{
    (a) Number of iterations and (b) total computation time with different $T$ and fixed $d=20$, $M=200$; (c)-(d) are similar but with different $d$ and fixed $T=20$ , $M=200$.
    Each set of parameters is tested with 10 different random seeds and the plots show both means and variances.
    }
    \label{fig:comparision}
\end{figure}

\begin{figure}[h]
    \centering
    \begin{subfigure}[b]{0.24\textwidth}
        \centering
        \includegraphics[width=1.0\textwidth]{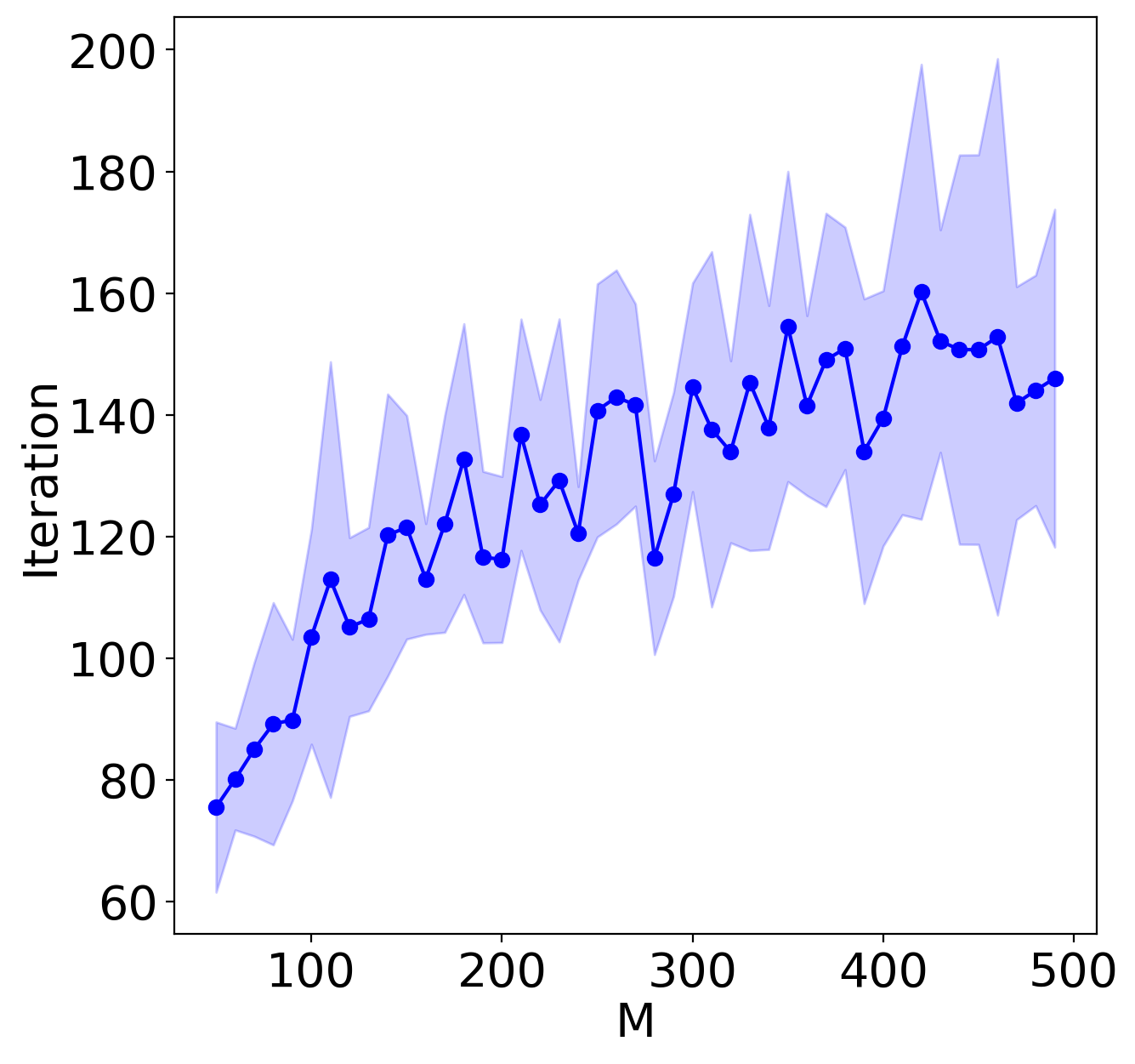}
        \caption{}
        \label{fig:M_step}
    \end{subfigure}
    \begin{subfigure}[b]{0.24\textwidth}
        \centering
        \includegraphics[width=1.0\textwidth]{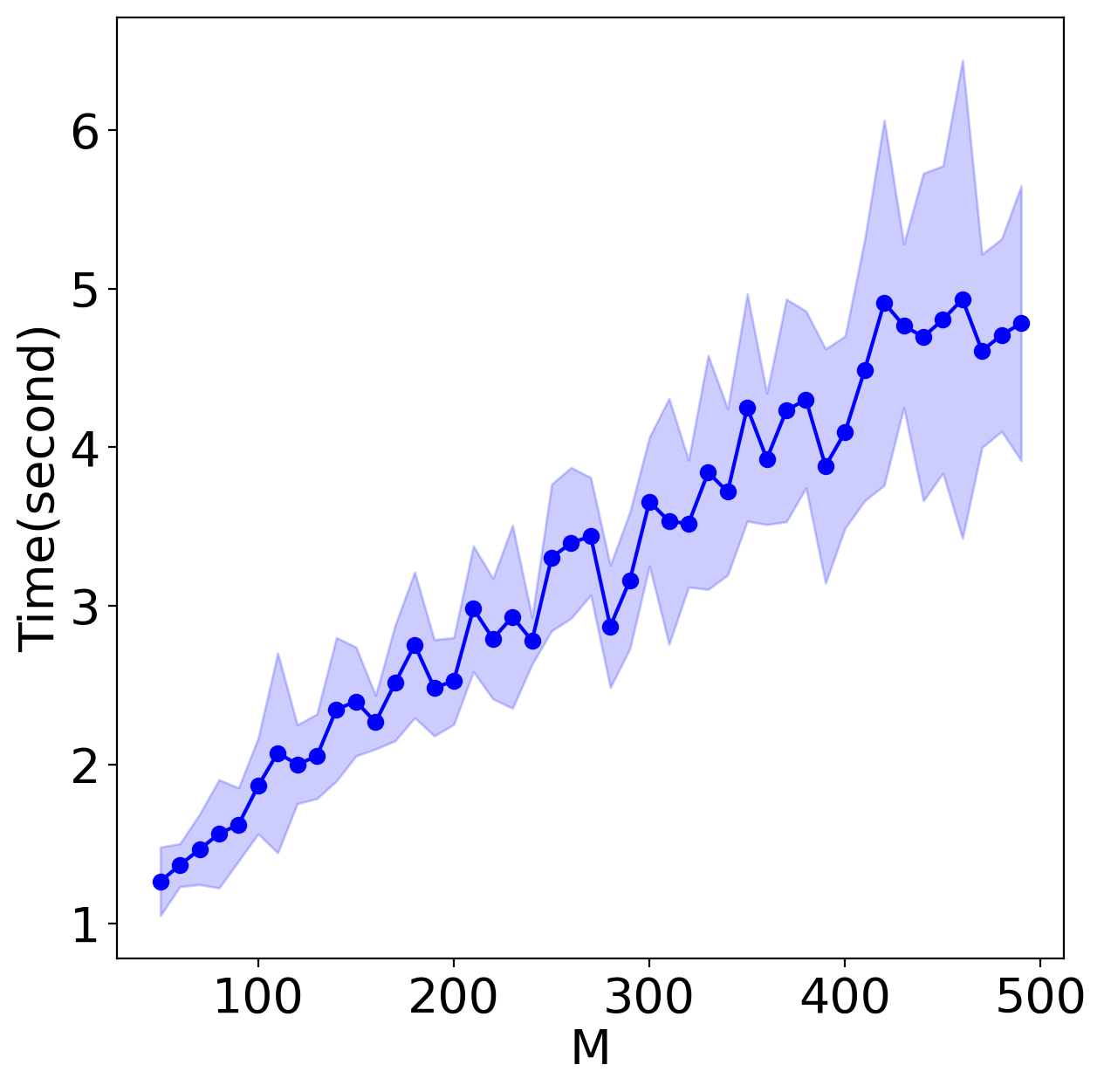}
        \caption{}
        \label{fig:M_time}
    \end{subfigure}
    \caption{
    (a) Number of iterations and (b) total computation time with increasing $M$ under fixed $d=20, T=20$ HMM graph.
    }
    \label{fig:comparision_M}
\end{figure}

\section{Conclusion} \label{sec:conclusion}
In this paper, we proposed an efficient algorithm to solve aggregate inference problems for HMMs with continuous observations. This algorithm is based on the collective forward-backward \cite{SinHaaZha20} algorithm which was designed for aggregate inference for HMMs with discrete observations. The latter was extended to the setting with continuous observations in this work through a reformulation of a key step in the algorithm which enables updating using samples from the observations. 
Note that though our algorithm is developed for HMM models, the same idea can be generalized to a large class of graphical models. This idea will be explored as a part of our future research.

\bibliography{refs}{}
\bibliographystyle{IEEEtran}


\appendix

\subsection{Proof of Theorem~\ref{thm:co-forward-backward}}
\label{app:theorem2}
\begin{proof}
Introducing Lagrange multipliers $a_t(x_t), b_t(o_t), c_t(x_t)$, $d_t(x_t),e_t(o_t),f_{1t}, f_{2t}$ for constraints Equation \eqref{eq:update_cons_tilde_tt_t},\eqref{eq:cons_tilde_tt_o}, \eqref{eq:cons_tt_t}, \eqref{eq:update_cons_simplex}, we arrive at the Lagrangian
    \begin{equation}
        \begin{aligned}
            \mathcal{L} = & \mathcal{F}(\bn) + \sum_{t=1}^{T} \sum_{x_t} a_t(x_t)[\int_{\cO} \tilde{n}_{tt}(x_t,o_t) do_t - n_t(x_t)] \\
& + \sum_{t=1}^{T} \int_{\cO} b_t(o_t)[\sum_{x_t} \tilde{n}_{tt}(x_t,o_t) - \tilde{n}_t(o_t)]do_t           \\
                          & + \sum_{t=1}^{T-1} \sum_{x_t} c_{t+1}(x_{t+1})[\sum_{x_t} n_{tt}(x_t,x_{t+1}) - n_{t+1}(x_{t+1})]         \\
                          & + \sum_{t=1}^{T-1} \sum_{x_t} d_{t}(x_{t})[\sum_{x_{t+1}} n_{tt}(x_t,x_{t+1}) - n_{t}(x_{t})]             \\
                          & + \sum_{t=1}^{T} \int_{\cO} e_t(o_t)[\tilde{n}_t(o_t)- y_t(o_t)]do_t                                      \\
                          & + \sum_{t=1}^{T} \{ f_{1t}[\sum_{x_t} n_{t}(x_t) -1] + f_{2t}[\int_{\cO} \tilde{n}_t(o_t)do_t -1] \}.
        \end{aligned}
    \end{equation}
    Setting the derivatives of the Lagrangian with respect to the $\{\bn_t,\tilde{\bn}_t,\bn_{tt}, \tilde{\bn}_{tt}\}$ to zero, we obtain the optimality conditions
    \begin{subequations}
        \begin{align}
             & \begin{cases}
                - 1 - \log n_t(x_t) - a_t(x_t) - d_t(x_t) + f_{1t}= 0,t=1           \\
                -1 - \log n_t(x_t) - a_t(x_t) - c_t(x_t) + f_{1t}= 0,t = T \\
                -2 -2\log n_t(x_t) -a_t(x_t) - c_t(x_t) - d_t(x_t) + f_{1t}= 0, \\
                \hspace{150pt minus 1fil} \quad \quad \quad t=2\cdots T-1  \hfilneg
\end{cases} \label{eq:lag_n_t}\\
& -b_t(o_t) + e_t(o_t) + f_{2t} = 0\label{eq:lag_tilde_n_t}\\
& \tilde{n}_{tt}(x_t,o_t) = p(o_t | x_t)e^{-a(x_t)-b(o_t)} \label{eq:lag_n_tt}   \\
& n_{tt}(x_t,x_{t+1}) = p(x_{t+1} | x_t)e^{-c(x_{t+1})-d(x_t)} \label{eq:lag_tilde_n_tt}.
        \end{align}
    \end{subequations}
Define, for all $t=1,2,\cdots,T$,
    \begin{subequations} \label{eq:lag_define}
        \begin{align}
            \alpha_t(x_t) = \sum_{x_{t-1}} p(x_t | x_{t-1}) e^{-d_t(x_{t-1})} \label{eq:define_alpha} \\
            \beta_t(x_t) = \sum_{x_{t+1}} p(x_{t+1} | x_t) e^{-c_t(x_{t+1})}  \label{eq:define_beta} \\
            \xi_t(o_t) = \sum_{x_t} p(o_t | x_t) e^{-a_t(x_t)}  \label{eq:define_xi} \\
            \gamma_t(x_t) = \int_{\cO} p(o_t | x_t)e^{-b_t(o_t)}do_t \label{eq:define_gamma}. 
        \end{align}
    \end{subequations}
Next we present the case with $t=2,\cdots, T-1$ in \eqref{eq:lag_n_t}; the other two cases with $t=1$ or $t=T$ can be analyzed similarly. It follows directly from \eqref{eq:lag_n_t} that
    \begin{equation}\label{eq:nt}
        n_t(x_t) \propto e^{-\frac{a_t(x_t)+c_t(x_t)+d_t(x_t)}{2}}.
    \end{equation}
Plugging \eqref{eq:nt} and \eqref{eq:lag_tilde_n_tt} into $n_t(x_t)n_t(x_t) = [\sum_{x_{t-1}}n_{t-1,t-1}(x_{t-1},x_t)][\sum_{x_{t+1}}n_{t t}(x_{t},x_{t+1})]$, in view of \eqref{eq:lag_define}, we arrive at
    \begin{equation} \label{eq:exp_a}
        e^{-a_t(x_t)}  \propto \alpha_t(x_t) \beta_t(x_t).
    \end{equation}
Plugging \eqref{eq:nt}, \eqref{eq:define_xi} and \eqref{eq:lag_tilde_n_tt} into $n_t(x_t)n_t(x_t) = [\sum_{x_{t+1}}n_{t t}(x_{t},x_{t+1})][\int_{\cO} \tilde{n}_{tt}(x_t,o_t)do_t]$, in view of  \eqref{eq:lag_define}, we get
    \begin{equation} \label{eq:exp_c}
        e^{-c_t(x_t)} \propto \beta_t(x_t) \gamma_t(x_t).
    \end{equation}
Similarly, by $n_t(x_t)n_t(x_t) = [\sum_{x_{t-1}}n_{t-1,t-1}(x_{t-1},x_t)]$
$[\int_{\cO} \tilde{n}_{tt}(x_t,o_t)do_t]$ we obtain
    \begin{equation} \label{eq:exp_d}
        e^{-d_t(x_t)} \propto \alpha_t(x_t) \gamma_t(x_t).
    \end{equation}
 
Clearly, \eqref{eq:marginals_hmm} is a direct consequence of \eqref{eq:nt}-\eqref{eq:exp_a}-\eqref{eq:exp_c}-\eqref{eq:exp_d} by
    \[
        n_t(x_t) \propto e^{\frac{-a_t(x_t)-c_t(x_t)-d_t(x_t)}{2}} \propto \alpha_t(x_t)\beta_t(x_t)\gamma_t(x_t).
    \]
Moreover, \eqref{eq:forward_backward1} follows by combining \eqref{eq:exp_d} and \eqref{eq:define_alpha}, \eqref{eq:forward_backward2} follows by combining \eqref{eq:exp_c} and \eqref{eq:define_beta}, and \eqref{eq:forward_backward4} follows by combining \eqref{eq:exp_a} and \eqref{eq:define_xi}.
Finally, by \eqref{eq:lag_n_tt} and in view of \eqref{eq:define_xi}
    \begin{equation}\label{eq:exp_b}
        e^{-b_t(o_t)} = \frac{\sum_{x_t} \tilde{n}_{tt}(x_t,o_t)}{\sum_{x_t} p(o_t|x_t)e^{-a_t(x_t)}} \propto \frac{y_t(o_t)}{\xi_t(o_t)},
    \end{equation}
where in the last step we have utilized the constraint $\sum_{x_t} \tilde{n}_{tt}(x_t,o_t) = y_t(o_t)$. The update \eqref{eq:forward_backward3} then follows by plugging \eqref{eq:exp_b} into \eqref{eq:define_gamma}.
\end{proof}

\end{document}